%% file: main.tex
\documentclass{article}
\usepackage{blindtext}
\usepackage[a4paper, total={6in, 8in}]{geometry}

\usepackage{times}
\usepackage{helvet}
\usepackage{courier}
\usepackage{graphicx}
\usepackage{array}
\usepackage{amsmath}
\usepackage{algorithm}
\usepackage{algpseudocode}
\usepackage{amsfonts}
\usepackage{graphicx}
\usepackage{multirow}

\usepackage{tikz}
\usetikzlibrary{shapes.geometric, arrows}
\usepackage[T1]{fontenc}
\usepackage[latin9]{inputenc}
\usepackage{mathrsfs}
\usepackage{amsmath}
\usepackage{amssymb}
\usepackage{graphicx}
\usepackage{braket}
\usepackage{color}
\usepackage{natbib}
\usepackage{dsfont}
\usepackage{mdframed}
\usepackage{amsmath, amssymb} 
\usepackage{bm} 
\usepackage{url}
\usepackage{breakurl}
\usepackage[breaklinks]{hyperref}
\usepackage{amsthm}
\usepackage{natbib} 
\usepackage{mathtools} 
\usepackage{booktabs} 
\usepackage{tikz}
\usetikzlibrary{shapes.geometric, arrows, positioning}

\newtheorem{lemma}{Lemma}

\begin{document}

\title{A Dynamical Systems Approach to Mitigating Oversmoothing in Graph Neural Networks}
\author{Biswadeep Chakraborty, Harshit Kumar, Saibal Mukhopadhyay\\
Department of Electrical and Computer Engineering\\
Georgia Institute of Technology
\date{}
}
\maketitle
\begin{abstract}
\begin{quote}
Oversmoothing in Graph Neural Networks (GNNs) poses a significant challenge as network depth increases, leading to homogenized node representations and a loss of expressiveness. In this work, we approach the oversmoothing problem from a dynamical systems perspective, providing a deeper understanding of the stability and convergence behavior of GNNs. Leveraging insights from dynamical systems theory, we identify the root causes of oversmoothing and propose \textbf{\textit{DYNAMO-GAT}}. This approach utilizes noise-driven covariance analysis and Anti-Hebbian principles to selectively prune redundant attention weights, dynamically adjusting the network's behavior to maintain node feature diversity and stability. Our theoretical analysis reveals how DYNAMO-GAT disrupts the convergence to oversmoothed states, while experimental results on benchmark datasets demonstrate its superior performance and efficiency compared to traditional and state-of-the-art methods. DYNAMO-GAT not only advances the theoretical understanding of oversmoothing through the lens of dynamical systems but also provides a practical and effective solution for improving the stability and expressiveness of deep GNNs.

\end{quote}
\end{abstract}

\begin{figure*}
    \centering
    \includegraphics[width=0.9\linewidth]{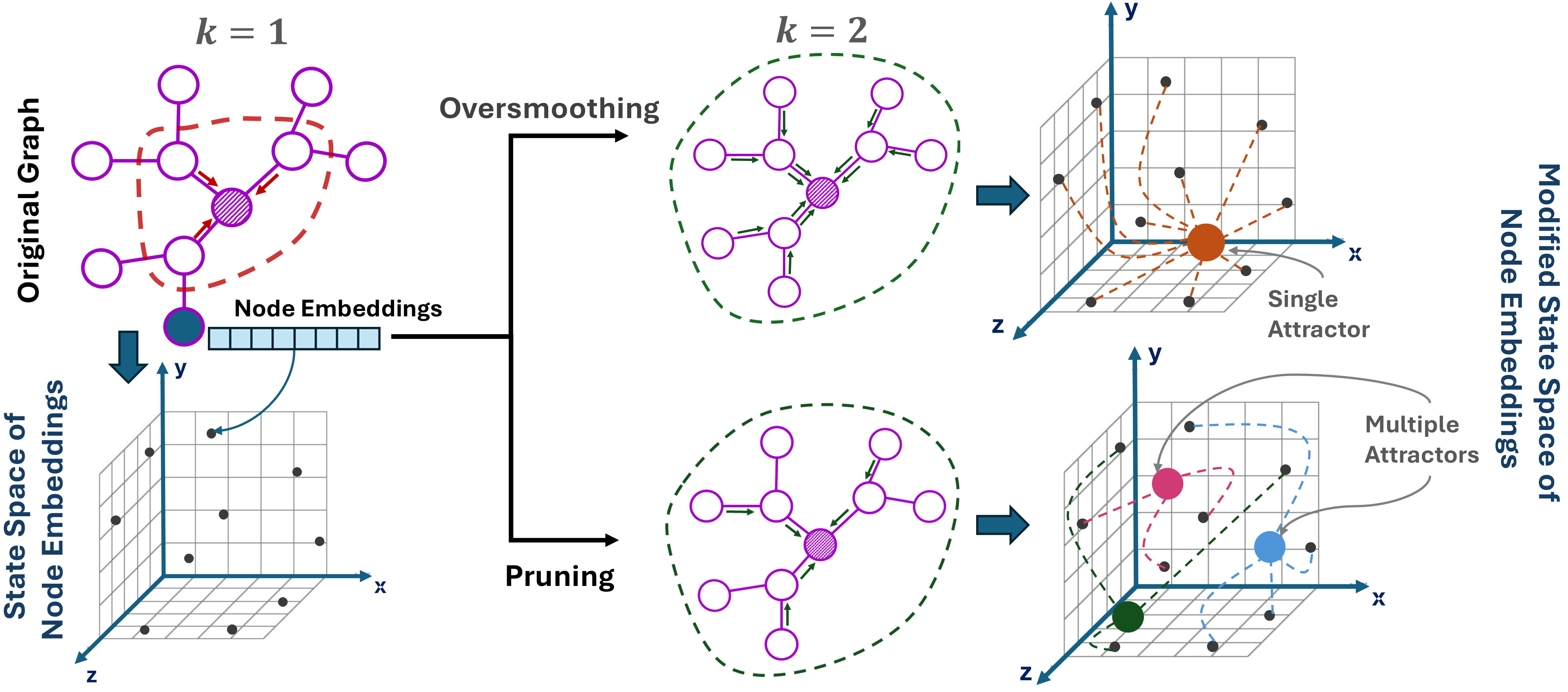}
    \caption{As the number of layers $k$ in a GNN increases, oversmoothing causes node embeddings to converge towards a single attractor state, resulting in the loss of node feature diversity. Pruning mitigates this effect by maintaining multiple attractor states, thereby preserving the distinctiveness of node embeddings and preventing the detrimental effects of oversmoothing.}
    \label{fig:overall}
\end{figure*}

\section{Introduction}

Graph Neural Networks (GNNs) \cite{wu2020comprehensive} have emerged as an important component in contemporary machine learning, excelling in tasks that require the analysis of graph-structured data. Their capacity to model complex relationships between nodes and edges has driven their widespread application in fields ranging from molecular property prediction \cite{gilmer2017neural, reiser2022graph, gasteiger2021gemnet} to social network analysis \cite{Kipf2017SemiSupervisedCW, fan2019graph} and recommendation systems \cite{ying2018graph}. However, one significant challenge that GNNs face is the phenomenon known as \textit{oversmoothing}. As the depth of the GNN increases, node representations tend to homogenize, leading to a decline in the network's ability to differentiate between nodes, ultimately impairing performance \cite{li2018deeper}.

Oversmoothing in GNNs has been extensively studied, with early works such as \citet{li2018deeper} identifying it as a critical issue in deep architectures like Graph Convolutional Networks (GCNs). Subsequent theoretical analyses \cite{Oono2019GraphNN, Cai2020ANO, Keriven2022NotTL, chen2020measuring, xu2019graph} have confirmed that oversmoothing is a fundamental problem in message-passing architectures, where repeated aggregation leads to the homogenization of node features. To counteract this, various strategies have been proposed, such as residual connections and skip connections \cite{li2019deepgcns, xu2018representation}, normalization methods \cite{ba2016layer, ioffe2015batch, zhou2020towards}, and attention mechanisms \cite{Velickovic2018graph}. However, these approaches primarily involve architectural modifications that do not fundamentally address the propagation dynamics responsible for oversmoothing. Consequently, they often provide only partial solutions, with oversmoothing persisting in deeper networks, ultimately limiting the expressiveness and effectiveness of deep GNNs in complex tasks.

More recent efforts, including pruning techniques aimed at enhancing efficiency and reducing model complexity, have also been explored as a means to address oversmoothing \cite{li2019deepgcns, zhao2020sparsity}. Pruning typically involves the removal of redundant edges or neurons, but these methods generally overlook the dynamical aspects of the network that contribute to oversmoothing. Similarly, while attention mechanisms, like GATs, have improved the focus on relevant parts of the graph, they remain susceptible to oversmoothing without proper regulation \cite{wu2023demystifying}.

In contrast to these existing methods, our work proposes a novel perspective by framing oversmoothing as a dynamical systems problem. Existing techniques often focus on structural adjustments without considering the underlying convergence behavior of GNNs, which can result in suboptimal solutions that fail to prevent oversmoothing in deeper networks \cite{li2018deeper, chen2020measuring, Oono2019GraphNN}. This limited focus can cause inefficiencies, where networks either become too complex or require excessive computational resources, yet still suffer from a loss of node feature diversity. Additionally, such methods may not generalize well across different graph structures, leading to inconsistent performance and limiting the applicability of GNNs in more challenging or real-world scenarios.

By viewing the iterative message-passing process in GNNs as analogous to a system converging towards a low-dimensional attractor \cite{li2019deepgcns}, we can pinpoint the exact mechanisms that lead to oversmoothing, where node representations collapse into indistinguishable states. This dynamical systems perspective allows us to move beyond superficial architectural fixes by directly addressing the root cause of oversmoothing: the convergence behavior of the network. By analyzing the system's stability and the conditions under which nodes lose their distinctiveness, we can design more effective strategies that prevent this collapse, ensuring that GNNs maintain expressiveness and perform well, even as the network depth increases.

The dynamical systems framework offers a new lens through which to study the convergence behavior of GNNs. Unlike previous approaches that treat oversmoothing as an empirical challenge to be mitigated through trial and error, our approach identifies the precise mechanisms that lead to oversmoothing. By analyzing the eigenvalues of the graph attention mechanism and studying the system's fixed points \cite{abbe2020entrywise, allen2019convergence}, we can characterize the conditions under which node representations lose their distinctiveness. This theoretical foundation explains existing empirical observations and opens new avenues for developing robust and adaptive GNN architectures.

Building on this foundation, we introduce DYNAMO-GAT, a novel GNN architecture specifically designed to counteract oversmoothing. Unlike existing methods that patch the oversmoothing issue through network adjustments, our approach delves into the root cause by analyzing the GNN's behavior as a dynamical system. This allows DYNAMO-GAT to adaptively adjust the network's behavior during training, ensuring that node representations remain diverse and informative across all layers. Our approach not only mitigates oversmoothing but also enhances the overall expressiveness and performance of deep GNNs. The main contributions of this paper are as follows:
\begin{itemize}
    \item We propose a novel theoretical framework for analyzing oversmoothing in GNNs using concepts from dynamical systems theory. This framework allows us to precisely identify the conditions under which oversmoothing occurs and provides insights into how it can be mitigated.
    \item We present DYNAMO-GAT, a GNN architecture that adaptively prevents oversmoothing during training by leveraging dynamical systems principles, offering a departure from static architectural modifications.
    \item We provide theoretical proofs that demonstrate how DYNAMO-GAT alters the system's fixed points, preserving node representation diversity even in deep networks.
\end{itemize}

Our work bridges the gap between empirical solutions to oversmoothing and a more fundamental understanding of the problem. By approaching oversmoothing from a dynamical systems perspective, we pave the way for more stable and expressive deep GNNs, addressing a critical gap in the current literature.

\section{Dynamical Systems View of Oversmoothing}

Oversmoothing in GNNs is a critical challenge, particularly as the depth of these networks increases. While Graph Attention Networks (GATs) introduce dynamic weighting mechanisms that can mitigate oversmoothing to some extent, they can also contribute to it under certain conditions \cite{Velickovic2018graph, Rusch2023SurveyOS}. To fully understand and address this phenomenon, we adopt a dynamical systems perspective \cite{Roth2024simplifying}.

Unlike traditional approaches that focus on architectural modifications, the dynamical systems view provides a more fundamental explanation by examining the stability and convergence properties of GNNs. By modeling GATs as dynamical systems, we can analyze how node representations evolve across layers and identify the conditions under which oversmoothing occurs \cite{Rusch2023demystifying, Giovanni2023oversquashing}. This perspective not only deepens our theoretical understanding but also suggests new strategies for mitigating oversmoothing \cite{Roth2024simplifying, Rusch2022graphCON}.

\textbf{GATs as Dynamical Systems.} In GATs, node representations evolve according to the learned attention weights $\alpha_{ij}$, which govern the influence of neighboring nodes. This dynamic weighting introduces complexity into the system's behavior, making it essential to understand how these weights evolve across layers. The attention mechanism can be seen as a time-varying update function:

\[
f(\mathbf{x}(t)) = \sigma\left(\sum_{j \in \mathcal{N}(i)} \alpha_{ij}(t) \mathbf{W} \mathbf{x}_j(t)\right),
\]

where the weights $\alpha_{ij}(t)$ are learned at each layer $t$. As these weights evolve, the system's dynamics may lead to a state where node representations become indistinguishable, resulting in oversmoothing. Understanding this process is key to designing GATs that avoid oversmoothing while still leveraging attention mechanisms effectively.

\subsection{Theoretical Analysis of Oversmoothing in GATs}

In this subsection, we rigorously analyze the phenomenon of oversmoothing in GATs using dynamical systems theory. Specifically, we explore the existence of fixed points, their stability, and the conditions under which node representations converge to indistinguishable states. The detailed theoretical proofs are given in the Supplementary Section 

\begin{lemma}
\label{lemma:fixed_points}
\textbf{(Existence of Fixed Points in GATs).} 
Consider a GAT with an update rule \( f: \mathbb{R}^{N \times d} \rightarrow \mathbb{R}^{N \times d} \), where \( \mathbf{X}(t) \in \mathbb{R}^{N \times d} \) represents the node features at layer \( t \). The update rule for the node features can be expressed as:

\[
\mathbf{X}_i(t+1) = f(\mathbf{X}_i(t)) = \sigma\left(\sum_{j \in \mathcal{N}(i)} \alpha_{ij}(t) \mathbf{W} \mathbf{X}_j(t)\right),
\]

where \( \sigma \) is a nonlinear activation function, \( \mathbf{W} \) is a weight matrix, and \( \alpha_{ij}(t) \) are the attention coefficients at layer \( t \) for node \( i \) and its neighbor \( j \).

If the spectral radius \( \rho(\mathbf{A}_{\text{eff}}(t)) \leq 1 \), the system will converge to a fixed point \( \mathbf{X}^* \in \mathbb{R}^{N \times d} \) such that \( \mathbf{X}^* = f(\mathbf{X}^*) \), indicating that node features become stable and indistinguishable, leading to oversmoothing.
\end{lemma}

\textbf{Proof Sketch:} The existence of fixed points follows from the contraction mapping principle. If the spectral radius \( \rho(\mathbf{A}_{\text{eff}}(t)) \leq 1 \), the GAT's update rule acts as a contractive operator on the node features. Repeated application of the update rule drives the node representations toward a stable fixed point \( \mathbf{X}^* \), leading to oversmoothing. \hfill \(\square\)

\begin{lemma}
\label{lemma:oversmoothing_attractor}
\textbf{(Oversmoothing as Convergence to an Attractor in GATs).} 
Let \( \mathbf{X}(t) \in \mathbb{R}^{N \times d} \) represent the node features at layer \( t \) in a GAT. The evolution of the node features is governed by the update rule:

\[
\mathbf{X}_i(t+1) = f(\mathbf{X}_i(t)) = \sigma\left(\sum_{j \in \mathcal{N}(i)} \alpha_{ij}(t) \mathbf{W} \mathbf{X}_j(t)\right).
\]

Assume that the update function \( f \) is a contraction mapping, i.e., there exists a constant \( c \in [0, 1) \) such that for all \( \mathbf{X}_1, \mathbf{X}_2 \in \mathbb{R}^{N \times d} \),

\[
\|f(\mathbf{X}_1) - f(\mathbf{X}_2)\| \leq c \|\mathbf{X}_1 - \mathbf{X}_2\|.
\]

Then, there exists an attractor \( \mathcal{A} \subseteq \mathbb{R}^{N \times d} \) such that:

\[
\lim_{t \to \infty} \mathbf{X}(t) \in \mathcal{A},
\]

where \( \mathcal{A} \) is a low-dimensional subspace in which node features become indistinguishable.
\end{lemma}

\textbf{Proof Sketch:} By the Banach fixed-point theorem, a contraction mapping on a complete metric space has a unique fixed point or attractor \( \mathcal{A} \). Thus, as \( t \to \infty \), the sequence \( \mathbf{X}(t) \) converges to \( \mathcal{A} \), a low-dimensional subspace where \( \mathbf{X}(t+1) = f(\mathbf{X}(t)) \). This implies oversmoothing, as node features collapse within this subspace. \hfill \(\square\)

\begin{lemma}
\label{lemma:stability}
\textbf{(Stability of Fixed Points in GATs).} 
For a GAT with an update rule \( f: \mathbb{R}^{N \times d} \rightarrow \mathbb{R}^{N \times d} \), let \( \mathbf{X}^* \in \mathbb{R}^{N \times d} \) be a fixed point of \( f \), i.e., \( f(\mathbf{X}^*) = \mathbf{X}^* \). The fixed point \( \mathbf{X}^* \) corresponds to oversmoothing if \( \lim_{t \rightarrow \infty} \|\mathbf{X}^*_i - \mathbf{X}^*_j\| = 0 \) for all \( i,j \). The fixed point \( \mathbf{X}^* \) is stable if and only if the spectral radius \( \rho(J_f(\mathbf{X}^*)) \leq 1 \).
\end{lemma}

\textbf{Proof Sketch:} Consider the state vector \( \mathbf{X}(t) \in \mathbb{R}^{N \times d} \) at layer \( t \), with the GAT update rule defined by \( \mathbf{X}(t+1) = f(\mathbf{X}(t)) \). To analyze the stability of the fixed point \( \mathbf{X}^* \), we linearize \( f \) around \( \mathbf{X}^* \) by considering a small perturbation \( \delta\mathbf{X}(t) = \mathbf{X}(t) - \mathbf{X}^* \). The update rule is approximated as:

\[
\delta\mathbf{X}(t+1) = J_f(\mathbf{X}^*)\delta\mathbf{X}(t),
\]

where \( J_f(\mathbf{X}^*) \) is the Jacobian matrix of \( f \) at \( \mathbf{X}^* \). The fixed point \( \mathbf{X}^* \) is stable if the perturbation \( \delta\mathbf{X}(t) \) decays over time, which occurs if and only if the spectral radius \( \rho(J_f(\mathbf{X}^*)) \leq 1 \). If \( \rho(J_f(\mathbf{X}^*)) > 1 \), the fixed point is unstable, and small perturbations will grow, leading to divergence from the fixed point. \hfill \(\square\)

\begin{figure*}[h]
    \centering
    \includegraphics[width=0.9\textwidth]{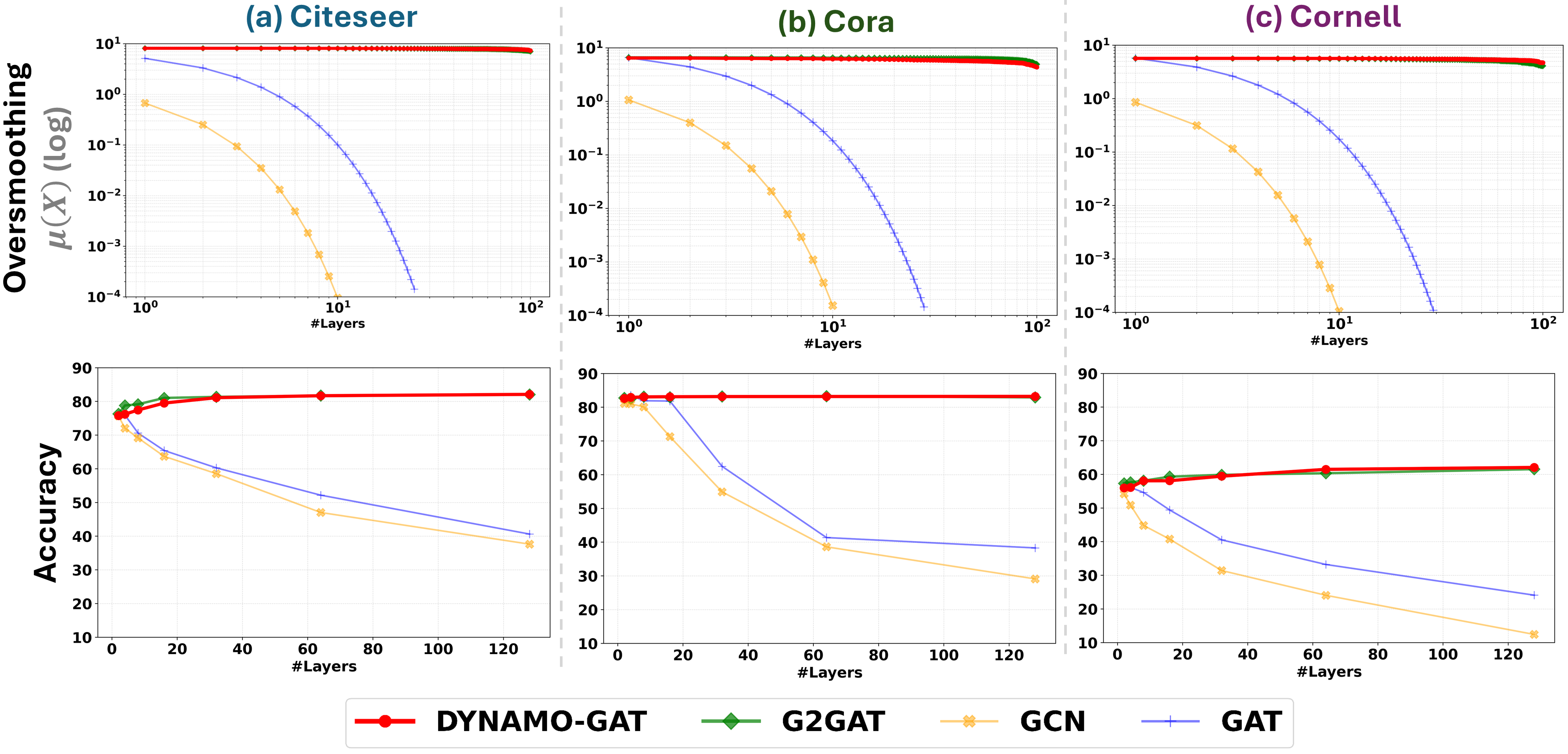}
    \caption{Comparison of oversmoothing coefficient (\(\mu(X)\)) and test accuracy across layers for Citeseer, Cora, and Cornell datasets. DYNAMO-GAT consistently outperforms both GCN, GAT and G2GAT maintaining high accuracy across all layers.}
    \label{fig:real_data}
\end{figure*}

\section{DYNAMO-GAT Algorithm}

The DYNAMO-GAT algorithm is a novel approach designed to address the oversmoothing problem in attention-based GNNs. It counters this by selectively pruning attention weights using a combination of noise injection, covariance analysis, Anti-Hebbian principles, dynamic thresholding, gradual pruning, and layer-wise pruning rates. The DYNAMO-GAT algorithm \ref{alg:dynamo} introduces non-linear perturbations into the system's state (node features) and modifies the connectivity structure (attention weights) dynamically. This not only disrupts the undesired fixed points associated with oversmoothing but also introduces mechanisms that ensure the system explores a richer set of node representations, maintaining diversity across layers.

\subsection{Covariance Matrix and Noise Injection}

The first step in the DYNAMO-GAT algorithm involves injecting independent Gaussian noise into the node features at each layer:
\begin{equation}
\mathbf{h}_i^{(l)} = \mathbf{h}_i^{(l)} + \sigma \xi_i^{(l)},
\end{equation}
where \( \xi_i^{(l)} \sim \mathcal{N}(0, I) \) represents Gaussian white noise with a standard deviation \( \sigma \). This noise perturbs the system state, revealing the underlying correlations between node features through their covariance structure.

The covariance matrix \( C^{(l)} \) is then computed as:
\begin{equation}
C_{ij}^{(l)} = \text{Cov}(\mathbf{h}_i^{(l)}, \mathbf{h}_j^{(l)}) = \mathbb{E}\left[(\mathbf{h}_i^{(l)} - \mathbb{E}[\mathbf{h}_i^{(l)}])(\mathbf{h}_j^{(l)} - \mathbb{E}[\mathbf{h}_j^{(l)}])^\top\right].
\end{equation}
This matrix captures the pairwise correlations between node features, which are crucial in identifying which connections (attention weights) contribute to oversmoothing. Nodes with highly correlated features are likely to converge towards similar representations. The covariance matrix measures the system's state coherence. High coherence (correlation) across many node pairs indicates a drift towards a stable, but undesirable, fixed point where oversmoothing dominates. By analyzing these correlations, DYNAMO-GAT can selectively target and prune connections that reinforce this drift, thereby altering the trajectory of the system's evolution.

\subsection{Anti-Hebbian Pruning Criterion}

The pruning strategy in DYNAMO-GAT is grounded in the Anti-Hebbian principle, which dictates that connections between highly correlated nodes should be weakened or eliminated. Taking inspiration from recent works on using noise to prune \cite{moore2020using, chakraborty2024sparse}, this principle computes the pruning probability \( p_{ij}^{(l)} \), which is dynamically adjusted based on a threshold \( \tau(t) \) that adapts to the distribution of edge weights. The dynamic pruning threshold \( \tau(t) \) is defined as:
\begin{equation}
\tau(t) = \mu(|w_{ij}|) + \beta \cdot \sigma(|w_{ij}|),
\end{equation}
where \( \mu \) and \( \sigma \) represent the mean and standard deviation of the edge weights, respectively. This threshold ensures that the pruning process is sensitive to the distribution of edge weights, allowing for more adaptive and context-sensitive pruning. The pruning probability is then computed as:
\begin{equation}
p_{ij}^{(l)} = r(t) \cdot \frac{|\alpha_{ij}^{(l)}|}{\tau(t)} \cdot (C_{ii}^{(l)} + C_{jj}^{(l)} \mp 2C_{ij}^{(l)}),
\end{equation}
where \( r(t) \) is the layer-wise pruning rate defined as \( r(t) = r_0 \cdot (1 + \gamma t) \). This scales with the depth of the layer, allowing for more aggressive pruning in later layers where oversmoothing is more likely to occur.

The pruning probability \( p_{ij}^{(l)} \) acts as a control mechanism that adjusts the strength and structure of the network's connections in response to the current state (as reflected by the covariance matrix). By dynamically adapting to the network's evolving state, DYNAMO-GAT effectively steers the system away from regions of the state space associated with oversmoothing, thus maintaining a more robust and diverse set of node representations.

\subsection{Gradual Pruning Process and Update Rule}

DYNAMO-GAT employs a gradual pruning approach, where edge weights are progressively reduced based on the computed pruning probability, rather than being immediately set to zero. This is given by:
\begin{equation}
w_{ij}(t+1) = w_{ij}(t) \cdot (1 - p_{ij}^{(l)}).
\end{equation}
An edge is fully pruned (i.e., its weight is set to zero) only if \( w_{ij}(t+1) \) falls below a small threshold \( \epsilon \).

The gradual pruning process introduces continuity into the network's dynamics, allowing the system to smoothly transition from one state to another. This contrasts with abrupt changes that could destabilize the learning process. The  gradual reduction of weights effectively modifies the original update rule \( F \) to a pruned update rule \( F_P \), which can be expressed as:
\begin{equation}
\mathbf{h}^{(l+1)} = F_P(\mathbf{h}^{(l)}, \alpha^{(l)}, \mathbf{W}^{(l)}, \mathbf{C}^{(l)}),
\end{equation}
where \( F_P \) incorporates the cumulative effects of pruning across layers. This gradual pruning can be seen as a form of perturbative adjustment, where the system is continuously nudged towards a more favorable configuration. The incremental changes introduced by gradual pruning helps the system avoid large, disruptive shifts that could lead to suboptimal convergence or loss of critical information.

\subsection{Recalibration of Attention Weights}

Once pruning has been applied, it is essential to recalibrate the remaining attention weights to ensure effective information propagation within the network. This recalibration process re-normalizes the attention coefficients \( \alpha_{ij}^{(l)} \) among the surviving connections:
\begin{equation}
\alpha_{ij}^{(l, \text{recal})} = \frac{\alpha_{ij}^{(l)}}{\sum_{k \in \mathcal{N}(i) \setminus \text{Pruned}(i)} \alpha_{ik}^{(l)}},
\end{equation}
where \( \text{Pruned}(i) \) denotes the set of pruned edges for node \( i \). Recalibration ensures that the information flow in the network remains balanced despite the reduced number of connections. This step is crucial for maintaining the stability of the network's dynamics post-pruning, as it prevents any remaining connections from becoming disproportionately influential, which could lead to oversmoothing.

\subsection{Theoretical Results}

Leveraging noise-driven covariance analysis, DYNAMO-GAT introduces stochasticity into the system, preventing the network from settling into fixed points prematurely. This stochasticity is particularly important in deeper networks, where oversmoothing is more likely to occur. The selective pruning mechanism further refines the system's dynamics, ensuring that only the most relevant connections are maintained, which aligns with the goal of avoiding low-dimensional attractors.

\begin{lemma}\label{lemma:pruning}
Let $G = (V, E)$ be a graph and $F: \mathbb{R}^{n \times d} \to \mathbb{R}^{n \times d}$ be the function representing the GNN layer transformation, where $n = |V|$ and $d$ is the feature dimension. Let $F_P: \mathbb{R}^{n \times d} \to \mathbb{R}^{n \times d}$ be the function representing the DYNAMO-GAT pruned GNN layer transformation. Denote by $\mathbf{X}^* \in \mathbb{R}^{n \times d}$ the oversmoothing fixed point such that $F(\mathbf{X}^*) = \mathbf{X}^*$. Let $J_F(\mathbf{X}^*) \in \mathbb{R}^{nd \times nd}$ and $J_{F_P}(\mathbf{X}^*) \in \mathbb{R}^{nd \times nd}$ be the Jacobian matrices of $F$ and $F_P$ respectively, evaluated at $\mathbf{X}^*$. Then:
\[
\rho(J_{F_P}(\mathbf{X}^*)) < \rho(J_F(\mathbf{X}^*))
\]
where $\rho(\cdot)$ denotes the spectral radius of a matrix. Consequently, DYNAMO-GAT pruning reduces the stability of the oversmoothing fixed point by decreasing the spectral radius of the Jacobian matrix of the pruned GNN.
\end{lemma}

\textbf{Short Proof:} DYNAMO-GAT pruning reduces the magnitude of the dominant eigenvalue \( \lambda_1 \) of \( J_{F_P} \), such that \( |\lambda_1^P| < |\lambda_1| \). Consequently, \( \rho(J_{F_P}) < \rho(J_F) \), leading to reduced stability of the oversmoothing fixed point. \hfill \(\square\)

This lemma shows that DYNAMO-GAT effectively reduces the stability of oversmoothing fixed points, preventing the network from converging to these undesirable states. By decreasing the spectral radius of the Jacobian matrix at these fixed points, DYNAMO-GAT modifies the system's dynamics to resist oversmoothing. \hfill \(\square\)

\begin{lemma}\label{lemma:diversity}
Let $G = (V, E)$ be a graph with $n = |V|$ nodes, and let $\mathbf{X}^{(t)} \in \mathbb{R}^{n \times d}$ be the matrix of node feature vectors at layer $t$ in a GNN. Define the covariance matrix $\mathbf{C}^{(t)} \in \mathbb{R}^{d \times d}$ of the node feature vectors at layer $t$ as:
\[
\mathbf{C}^{(t)} = \frac{1}{n} (\mathbf{X}^{(t)})^\top \mathbf{X}^{(t)} - \frac{1}{n^2} (\mathbf{X}^{(t)})^\top \mathbf{1}_n \mathbf{1}_n^\top \mathbf{X}^{(t)}
\]
where $\mathbf{1}_n \in \mathbb{R}^n$ is the vector of all ones. Then, DYNAMO-GAT algorithm ensures that:
\[
\text{rank}(\mathbf{C}^{(t)}) = d, \quad \forall t \in \{0, 1, \ldots, T\}
\]
where $T$ is the total number of layers in the GNN. Consequently, this preserves the full rank of $\mathbf{C}^{(t)}$ and prevents the collapse of node features into any subspace of dimension less than $d$.
\end{lemma}

\textbf{Short Proof:}
Let \( \mathbf{h}_i(t) \) denote the feature vector of node \( i \) at layer \( t \), and let the covariance matrix \( C(t) \) be defined as \( C(t) = \mathbb{E}[(\mathbf{h}(t) - \mathbb{E}[\mathbf{h}(t)])(\mathbf{h}(t) - \mathbb{E}[\mathbf{h}(t)])^\top] \). The DYNAMO-GAT algorithm selectively prunes edges that contribute to high covariance values, thereby reducing correlations between node features. This ensures that \( \text{rank}(C(t)) = d \), where \( d \) is the dimensionality of the feature space, preserving the full rank of \( C(t) \) and preventing feature collapse. \hfill \(\square\)

\begin{algorithm}[h]
\caption{Dynamic Nonlinear Anti-Hebbian GAT with Pruning Optimization (DYNAMO-GAT)}
\label{alg:dynamo}
\begin{algorithmic}[1]
\Require Graph \( G = (V, E) \), node features \( X \), noise level \( \sigma \), initial pruning constant \( K_0 \), layer-wise pruning rate \( r_0 \), adaptation parameters \( \beta \), pruning threshold \( \epsilon \)
\State Initialize node features: \( \mathbf{h}_i(0) = X_i \) for all \( i \in V \)
\For{each layer \( t \) in the GNN}
    \State Inject Gaussian noise: \( \mathbf{\xi}(t) \sim \mathcal{N}(0, I) \)
    \State Perturb node features: \( \mathbf{h}_i(t) = \mathbf{h}_i(t) + \sigma \mathbf{\xi}_i(t) \)
    \State Compute covariance matrix \( C \) based on noisy node features:
    \[
    C_{ij} = \mathbb{E}\left[ (\mathbf{h}_i(t) - \mathbb{E}[\mathbf{h}_i(t)])(\mathbf{h}_j(t) - \mathbb{E}[\mathbf{h}_j(t)])^\top \right]
    \]
    \State Compute pruning threshold: \( \tau(t) = \mu(|w_{ij}|) + \beta \cdot \sigma(|w_{ij}|) \)
    \For{each edge \( (i, j) \in E \)}
        \State Compute layer-wise pruning rate: \( r(t) = r_0 \cdot (1 + \gamma t) \)
        \State Update pruning probability \( p_{ij} \) using \( \tau(t) \) and \( r(t) \):
        \[
        p_{ij} =
        \begin{cases}
        r(t) \cdot \frac{w_{ij}}{\tau(t)} \cdot \left(C_{ii} + C_{jj} - 2C_{ij}\right), & \text{if } w_{ij} > 0, \\
        r(t) \cdot \frac{|w_{ij}|}{\tau(t)} \cdot \left(C_{ii} + C_{jj} + 2C_{ij}\right), & \text{if } w_{ij} < 0,
        \end{cases}
        \]
        \State Update edge weight: \( w_{ij}(t+1) = w_{ij}(t) \cdot (1 - p_{ij}) \)
        \State Prune edge if \( w_{ij}(t+1) < \epsilon \)
    \EndFor
    \State Recalibrate attention weights for remaining edges:
    \[
    \alpha_{ij}^{(t, \text{recal})} = \frac{\alpha_{ij}^{(t)}}{\sum_{k \in \mathcal{N}(i) \setminus \text{Pruned}(i)} \alpha_{ik}^{(t)}}
    \]
    \State Update node features: \( \mathbf{h}_i(t+1) = \sigma\left( \sum_{j \in \mathcal{N}(i)} w_{ij}(t+1) \mathbf{h}_j(t) \right) \)
\EndFor
\State \textbf{Output} Final node representations after pruning
\end{algorithmic}
\end{algorithm}

\begin{table}[]
\centering
\caption{Table comparing the different datasets and the number of GFLOPS for each model for each dataset}
\label{tab:gflops}
\resizebox{0.7\columnwidth}{!}{%
\begin{tabular}{cccccccccccccccc}
\hline
 & \textbf{Metric} & \textbf{Cora} & \textbf{Citeseer} & \textbf{Cornell} \\  \cline{2-5} 
 
& Nodes $(N)$ & 2,708 & 3,327 & 183 \\
\textbf{Models} & Edges $(E)$ & 5,429 & 4,732 & 280 \\
& Avg. Degree $(2|E|/|N|)$ & 4.01 & 2.84 & 3.06 \\ \hline \hline

\multirow{4}{*}{\textbf{GCN}} 
& Best Accuracy & 81.5 & 75.7 & 54.2 \\
& \#Layers & 2 & 2 & 2 \\
& GFLOPS & 0.598 & 1.789 & 0.049 \\
& Accuracy/GFLOPS & 136.28 & 43.53 & 1106.12 \\ \hline

\multirow{4}{*}{\textbf{GAT}} 
& Best Accuracy & 82.55 & 76.1 & 56.3 \\
& \#Layers & 4 & 2 & 2 \\
& GFLOPS & 2.351 & 6.754 & 0.184 \\
& Accuracy/GFLOPS & 35.11 & 11.27 & 306.52 \\ \hline

\multirow{4}{*}{\textbf{G2GAT}} 
& Best Accuracy & 83.27 & 82.06 & 61.55 \\
& \#Layers & 128 & 128 & 128 \\
& GFLOPS & 1.209 & 2.452 & 0.0879 \\
& Accuracy/GFLOPS & 68.88 & 33.47 & 700.34 \\ \hline

\multirow{4}{*}{\textbf{DYNAMO-GAT}} 
& Best Accuracy & 83.21 & 82.01 & 62.56 \\
& \#Layers & 128 & 128 & 128 \\
& GFLOPS & 0.605 & 1.675 & 0.051 \\
& Accuracy/GFLOPS & \textbf{137.53} & \textbf{48.96} & \textbf{1226.67} \\ \hline
\end{tabular}%
}
\end{table}

\section{Experimental Results}

\subsection{Experimental Setup}
\textbf{Datasets:}  
We conduct experiments on three real-world and two synthetic datasets. We use Cora \cite{mccallum2000automating}, Citeseer \cite{sen2008collective}, two citation networks, and Cornell \cite{cornell_gnn_dataset} which is part of the WebKB collection. For synthetic datasets, \texttt{'syn\_products'} \cite{zhu2020beyond} simulate product co-purchasing.

\textbf{Baselines:}   We compare DYNAMO-GAT with GCN \cite{Kipf2017SemiSupervisedCW}, GAT \cite{Velickovic2017GraphAN}, and G2GAT \cite{rusch2023gradient}, focusing on their effectiveness in preventing oversmoothing.

\textbf{Evaluation Metrics:}   Models are evaluated using Accuracy, Oversmoothing Coefficient (\(\mu\)) \cite{wu2023demystifying}, GFLOPS, and the Accuracy/GFLOPS ratio to gauge the trade-off between performance and computational cost.

\subsection{Performance on Real-World Datasets}

Figure~\ref{fig:real_data} illustrates the performance of DYNAMO-GAT, G2GAT, GCN, and GAT across three real-world datasets: Citeseer, Cora, and Cornell. The top row shows the oversmoothing coefficient (\(\mu(X)\)) on a log scale, while the bottom row displays the test accuracy as the number of layers increases.

\textbf{Oversmoothing Coefficient (\(\mu(X)\))} [Figs. ~\ref{fig:real_data}(a,b)]: The results demonstrate that GCN and GAT suffer from significant oversmoothing as the number of layers increases. Their oversmoothing coefficients decrease rapidly, indicating that node features become increasingly indistinguishable. G2GAT performs better by reducing the rate of oversmoothing, but it still shows a downward trend as layers increase. In contrast, DYNAMO-GAT maintains a nearly constant oversmoothing coefficient across all layers, effectively preventing this phenomenon. This stability suggests that DYNAMO-GAT preserves meaningful node representations even in deep architectures.

\textbf{Test Accuracy }[Figs. ~\ref{fig:real_data}(c,d)]: The test accuracy results align with the oversmoothing observations. GCN and GAT experience a sharp decline in accuracy as the number of layers increases, reflecting the negative impact of oversmoothing on model performance. G2GAT performs better, with a slower decline in accuracy, but still struggles as the network depth increases. DYNAMO-GAT, however, consistently achieves the highest accuracy across all datasets and depths. Its ability to maintain high accuracy even with many layers indicates that it effectively balances expressivity and resistance to oversmoothing.

These observations underscore the challenges of using deep GNNs in practical applications, where oversmoothing can severely degrade performance. The consistent performance of DYNAMO-GAT across different datasets and network depths suggests that it is a robust solution for deep GNNs, addressing a critical limitation of existing models. This makes DYNAMO-GAT particularly suitable for tasks that require deep networks without sacrificing accuracy or node representation quality.

\begin{figure*}[h]
    \centering
    \includegraphics[width=0.87\linewidth]{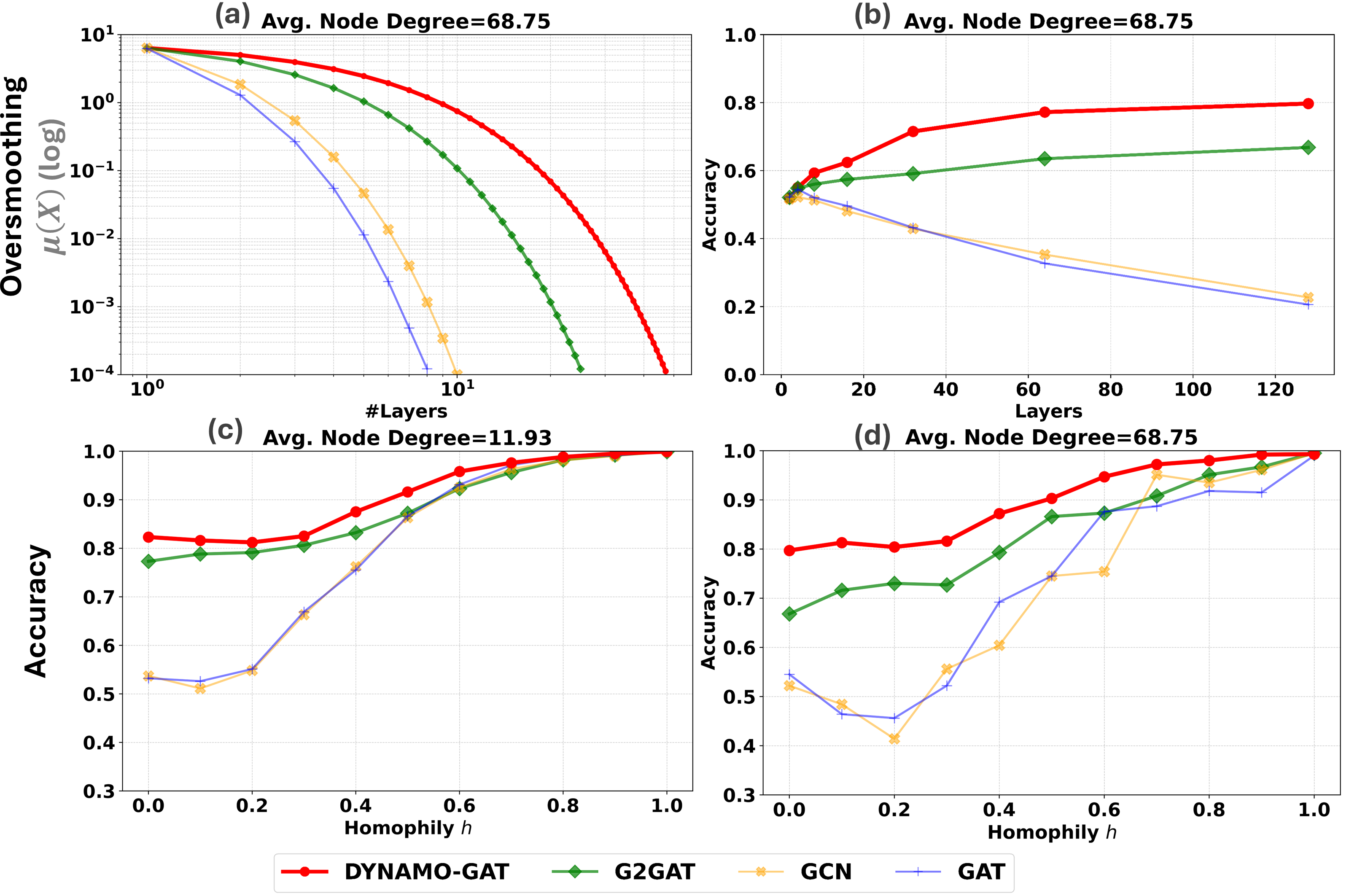}
    \caption{Performance of DYNAMO-GAT, G2GAT, GCN, and GAT on the Syn\_Products dataset. 
    (a) Oversmoothing vs. layers: DYNAMO-GAT shows the least oversmoothing.
     Comparing test accuracy (b)  vs. number of layers (c) vs. homophily for sparse graph (Avg. Degree=11.93) (d) vs. homophily for dense graph (Avg. Degree=68.75)}
    \label{fig:synthetic}
\end{figure*}

\subsection{Performance Comparison Across Datasets}

Table~\ref{tab:gflops} compares the performance of DYNAMO-GAT, G2GAT, GCN, and GAT across three datasets: Cora, Citeseer, and Cornell. The table highlights key metrics such as best accuracy, the number of layers, GFLOPS, and the accuracy-to-GFLOPS ratio.

\begin{itemize}
    \item \textbf{Best Accuracy:} DYNAMO-GAT consistently achieves the highest accuracy across all datasets, particularly excelling on the Cornell dataset with 62.56\%. This demonstrates its robustness in deep architectures.
    \item \textbf{GFLOPS:} Despite its deep architecture (128 layers), DYNAMO-GAT is computationally more efficient than GAT and G2GAT, with significantly lower GFLOPS, especially on larger datasets like Cora and Citeseer.
    \item \textbf{Accuracy/GFLOPS Ratio:} DYNAMO-GAT outperforms all models in the accuracy-to-GFLOPS ratio, indicating the best trade-off between accuracy and computational cost. For example, on Cora, it achieves 137.53, compared to GCN's 136.28 and GAT's 35.11.
\end{itemize}

\subsection{Synthetic Dataset Results}

Figure~\ref{fig:synthetic} presents the performance of DYNAMO-GAT, G2GAT, GCN, and GAT on the Syn\_Products dataset, tested across varying node degrees and homophily levels.

\textbf{(a) Oversmoothing vs. Layers (Avg. Degree = 68.75):} DYNAMO-GAT exhibits the least oversmoothing as layers increase, maintaining higher \(\mu(X)\) compared to other models. This indicates DYNAMO-GAT's robustness in preserving node features even in deep networks.

\textbf{(b) Accuracy vs. Layers (Avg. Degree = 68.75):} DYNAMO-GAT consistently achieves the highest accuracy across all layers, outperforming G2GAT, GCN, and GAT. This demonstrates its effectiveness in managing deep architectures without performance degradation.

 \textbf{(c) Accuracy vs. Homophily (Avg. Degree = 11.93):} In sparse graphs, DYNAMO-GAT and G2GAT perform well across all homophily levels, with DYNAMO-GAT showing stronger performance as homophily increases. This highlights its adaptability in different homophily settings.

 \textbf{(d) Accuracy vs. Homophily (Avg. Degree = 68.75):} In dense graphs, DYNAMO-GAT significantly outperforms other models, particularly in low-homophily settings, showcasing its strength in complex, heterophilic structures.

\textbf{Summary: } DYNAMO-GAT consistently outperforms other models in preventing oversmoothing and maintaining accuracy. Its efficiency, as highlighted by the accuracy-to-GFLOPS ratio, makes it suitable for real-world applications where computational resources are limited. The model's versatility across graph densities and homophily levels suggests it is well-suited for a range of tasks, from social network analysis to biological network modeling.

\section{Conclusion}

This paper introduced DYNAMO-GAT, a novel approach to addressing oversmoothing in GNNs through dynamical systems theory. By shifting from structural modifications to adaptive pruning guided by noise-driven covariance analysis, DYNAMO-GAT effectively prevents node representation collapse, supported by rigorous theoretical analysis and strong experimental performance on benchmark datasets. Beyond GNNs, the dynamical systems perspective introduced here may inspire new approaches to stability and expressiveness across various deep learning architectures, potentially driving innovations in AI. In summary, the dynamical systems based approach offers a powerful solution to the oversmoothing problem and sets the stage for future advancements in handling complex, large-scale graphs in deep learning models.

\bibliography{aaai}
\input{suppl}

\end{document}

%% file: suppl.tex
\newpage
\section{Supplementary Sections}
\section{Theoretical Proofs}

\subsection{Lemma 1}

\begin{lemma} Given a GAT with the update rule \( f: \mathbb{R}^{N \times d} \rightarrow \mathbb{R}^{N \times d} \), where node features at layer \( t \) are denoted by \( \mathbf{X}(t) \), and the update rule is:

\[
\mathbf{X}_i(t+1) = f(\mathbf{X}_i(t)) = \sigma\left(\sum_{j \in \mathcal{N}(i)} \alpha_{ij}(t) \mathbf{W} \mathbf{X}_j(t)\right),
\]

with \( \sigma \) as a nonlinear activation function, \( \mathbf{W} \) as a weight matrix, and \( \alpha_{ij}(t) \) as the attention coefficients.

The lemma asserts that if the spectral radius \( \rho(\mathbf{A}_{\text{eff}}(t)) \leq 1 \), the system will converge to a fixed point \( \mathbf{X}^* \) such that \( \mathbf{X}^* = f(\mathbf{X}^*) \).
\end{lemma}

 \begin{proof}

   To prove the existence of a fixed point, we utilize the contraction mapping theorem. A function \( f \) is a contraction if there exists a constant \( c \in [0, 1) \) such that for any two points \( \mathbf{X}, \mathbf{Y} \in \mathbb{R}^{N \times d} \),
   \[
   \|f(\mathbf{X}) - f(\mathbf{Y})\| \leq c \|\mathbf{X} - \mathbf{Y}\|.
   \]

   The spectral radius condition \( \rho(\mathbf{A}_{\text{eff}}(t)) \leq 1 \) implies that the effective adjacency matrix does not expand distances between points in the feature space. Specifically, when the spectral radius is strictly less than 1, the system contracts, ensuring convergence.   Given the update rule, the norm of the feature difference after one iteration can be bounded by:
   \begin{align}
          \|\mathbf{X}(t+1) - \mathbf{X}(t)\| &\leq \|\sigma\left(\sum_{j \in \mathcal{N}(i)} \alpha_{ij}(t) \mathbf{W} \mathbf{X}_j(t)\right) \nonumber \\
          &- \sigma\left(\sum_{j \in \mathcal{N}(i)} \alpha_{ij}(t) \mathbf{W} \mathbf{X}_j(t-1)\right)\|.
   \end{align}

   Assuming \( \sigma \) is Lipschitz continuous with a Lipschitz constant \( L \), the above difference can be bounded as:
   \[
   \leq L \|\mathbf{A}_{\text{eff}}(t)(\mathbf{X}(t) - \mathbf{X}(t-1))\|,
   \]
   where \( \mathbf{A}_{\text{eff}}(t) \) is the effective adjacency matrix incorporating the attention coefficients and weight matrix. If \( \rho(\mathbf{A}_{\text{eff}}(t)) \leq 1 \), then:
   \[
   \|\mathbf{X}(t+1) - \mathbf{X}(t)\| \leq c \|\mathbf{X}(t) - \mathbf{X}(t-1)\|,
   \]
   with \( c \leq L \), ensuring contraction for \( L < 1 \).

   By the Banach Fixed Point Theorem, since the mapping is a contraction, a unique fixed point \( \mathbf{X}^* \) exists such that:
   \[
   \mathbf{X}^* = f(\mathbf{X}^*).
   \]
   This fixed point represents the state where the features no longer change, leading to oversmoothing as all nodes converge to the same feature vector.
   As the number of iterations increases, the node features become indistinguishable, confirming that the system converges to a state of oversmoothing. The spectral radius condition \( \rho(\mathbf{A}_{\text{eff}}(t)) \leq 1 \) ensures that the system cannot escape this fixed point, reinforcing the oversmoothing phenomenon.

\end{proof}

\subsection{Lemma 2}

\begin{lemma}
\label{lemma:oversmoothing_attractor}
\textbf{(Oversmoothing as Convergence to an Attractor in GATs).} 
Let \( \mathbf{X}(t) \in \mathbb{R}^{N \times d} \) represent the node features at layer \( t \) in a GAT. The evolution of the node features is governed by the update rule:

\[
\mathbf{X}_i(t+1) = f(\mathbf{X}_i(t)) = \sigma\left(\sum_{j \in \mathcal{N}(i)} \alpha_{ij}(t) \mathbf{W} \mathbf{X}_j(t)\right).
\]

Assume that the update function \( f \) is a contraction mapping, i.e., there exists a constant \( c \in [0, 1) \) such that for all \( \mathbf{X}_1, \mathbf{X}_2 \in \mathbb{R}^{N \times d} \),

\[
\|f(\mathbf{X}_1) - f(\mathbf{X}_2)\| \leq c \|\mathbf{X}_1 - \mathbf{X}_2\|.
\]

Then, there exists an attractor \( \mathcal{A} \subseteq \mathbb{R}^{N \times d} \) such that:

\[
\lim_{t \to \infty} \mathbf{X}(t) \in \mathcal{A},
\]

where \( \mathcal{A} \) is a low-dimensional subspace in which node features become indistinguishable.
\end{lemma}

\begin{proof}
Consider the update rule for the node features in a GAT:

\[
\mathbf{X}_i(t+1) = f(\mathbf{X}_i(t)) = \sigma\left(\sum_{j \in \mathcal{N}(i)} \alpha_{ij}(t) \mathbf{W} \mathbf{X}_j(t)\right),
\]

where \( \sigma \) is a nonlinear activation function, \( \alpha_{ij}(t) \) are the attention coefficients, and \( \mathbf{W} \) is a weight matrix. Assume that \( f \) is a contraction mapping, i.e., there exists a constant \( c \in [0, 1) \) such that for all \( \mathbf{X}_1, \mathbf{X}_2 \in \mathbb{R}^{N \times d} \),

\[
\|f(\mathbf{X}_1) - f(\mathbf{X}_2)\| \leq c \|\mathbf{X}_1 - \mathbf{X}_2\|.
\]

By the Banach fixed-point theorem, since \( f \) is a contraction mapping, there exists a unique fixed point \( \mathbf{X}^* \in \mathbb{R}^{N \times d} \) such that:

\[
\mathbf{X}^* = f(\mathbf{X}^*).
\]

Define the sequence of node features \( \mathbf{X}(t) \) as:

\[
\mathbf{X}(t+1) = f(\mathbf{X}(t)),
\]

with the initial condition \( \mathbf{X}(0) \). Due to the contraction property, the sequence \( \mathbf{X}(t) \) converges to the fixed point \( \mathbf{X}^* \) as \( t \to \infty \):

\[
\lim_{t \to \infty} \|\mathbf{X}(t) - \mathbf{X}^*\| = 0.
\]

Let \( \mathcal{A} \subseteq \mathbb{R}^{N \times d} \) denote the attractor, which is the set of points to which the sequence \( \mathbf{X}(t) \) converges. Then,

\[
\lim_{t \to \infty} \mathbf{X}(t) \in \mathcal{A}.
\]

Since \( f \) is a contraction mapping, the attractor \( \mathcal{A} \) is a low-dimensional subspace where node features become indistinguishable, i.e., for any \( \mathbf{X}_1, \mathbf{X}_2 \in \mathcal{A} \),

\[
\|\mathbf{X}_1 - \mathbf{X}_2\| = 0.
\]

Thus, the node features converge to a low-dimensional subspace \( \mathcal{A} \), leading to oversmoothing:

\[
\lim_{t \to \infty} \mathbf{X}(t) \in \mathcal{A},
\]

where \( \mathcal{A} \) represents a low-dimensional space where all node features are similar or identical.
\end{proof}

\subsection{Lemma 3}

\begin{lemma}
\label{lemma:stability}
\textbf{(Stability of Fixed Points in GATs).} 
For a GAT with an update rule \( f: \mathbb{R}^{N \times d} \rightarrow \mathbb{R}^{N \times d} \), let \( \mathbf{X}^* \in \mathbb{R}^{N \times d} \) be a fixed point of \( f \), i.e., \( f(\mathbf{X}^*) = \mathbf{X}^* \). The fixed point \( \mathbf{X}^* \) corresponds to oversmoothing if \( \lim_{t \rightarrow \infty} \|\mathbf{X}^*_i - \mathbf{X}^*_j\| = 0 \) for all \( i,j \). The fixed point \( \mathbf{X}^* \) is stable if and only if the spectral radius \( \rho(J_f(\mathbf{X}^*)) \leq 1 \).
\end{lemma}

\begin{proof}
Consider the GAT update rule \( f: \mathbb{R}^{N \times d} \rightarrow \mathbb{R}^{N \times d} \), where \( \mathbf{X}^* \in \mathbb{R}^{N \times d} \) is a fixed point such that:

\[
f(\mathbf{X}^*) = \mathbf{X}^*.
\]

To analyze the stability of \( \mathbf{X}^* \), introduce a small perturbation \( \delta\mathbf{X}(t) = \mathbf{X}(t) - \mathbf{X}^* \). The update rule can be linearized around \( \mathbf{X}^* \) as follows:

\[
\delta\mathbf{X}(t+1) = J_f(\mathbf{X}^*)\delta\mathbf{X}(t),
\]

where \( J_f(\mathbf{X}^*) \) is the Jacobian matrix of \( f \) at the fixed point \( \mathbf{X}^* \).

The perturbation \( \delta\mathbf{X}(t) \) will decay over time, indicating that \( \mathbf{X}^* \) is stable if and only if the spectral radius of the Jacobian matrix \( \rho(J_f(\mathbf{X}^*)) \) satisfies:

\[
\rho(J_f(\mathbf{X}^*)) \leq 1.
\]

If \( \rho(J_f(\mathbf{X}^*)) > 1 \), the perturbation \( \delta\mathbf{X}(t) \) will grow over time, leading to instability and divergence from the fixed point \( \mathbf{X}^* \).

Finally, the fixed point \( \mathbf{X}^* \) corresponds to oversmoothing if:

\[
\lim_{t \rightarrow \infty} \|\mathbf{X}^*_i - \mathbf{X}^*_j\| = 0 \quad \text{for all } i, j.
\]

Thus, the stability of the fixed point is directly related to the spectral radius \( \rho(J_f(\mathbf{X}^*)) \). If \( \rho(J_f(\mathbf{X}^*)) \leq 1 \), the fixed point is stable, and oversmoothing occurs as node features converge to the same value. Conversely, if \( \rho(J_f(\mathbf{X}^*)) > 1 \), the fixed point is unstable, and oversmoothing does not occur.

\end{proof}

\subsection{Lemma 4}

\begin{lemma}\label{lemma:pruning}
Let $G = (V, E)$ be a graph and $F: \mathbb{R}^{n \times d} \to \mathbb{R}^{n \times d}$ be the function representing the GNN layer transformation, where $n = |V|$ and $d$ is the feature dimension. Let $F_P: \mathbb{R}^{n \times d} \to \mathbb{R}^{n \times d}$ be the function representing the DYNAMO-GAT pruned GNN layer transformation. Denote by $\mathbf{X}^* \in \mathbb{R}^{n \times d}$ the oversmoothing fixed point such that $F(\mathbf{X}^*) = \mathbf{X}^*$. Let $J_F(\mathbf{X}^*) \in \mathbb{R}^{nd \times nd}$ and $J_{F_P}(\mathbf{X}^*) \in \mathbb{R}^{nd \times nd}$ be the Jacobian matrices of $F$ and $F_P$ respectively, evaluated at $\mathbf{X}^*$. Then:
\[
\rho(J_{F_P}(\mathbf{X}^*)) < \rho(J_F(\mathbf{X}^*))
\]
where $\rho(\cdot)$ denotes the spectral radius of a matrix. Consequently, DYNAMO-GAT pruning reduces the stability of the oversmoothing fixed point by decreasing the spectral radius of the Jacobian matrix of the pruned GNN.
\end{lemma}

\begin{proof}
Let \( G = (V, E) \) be a graph and \( F: \mathbb{R}^{n \times d} \to \mathbb{R}^{n \times d} \) represent the GNN layer transformation, where \( n = |V| \) is the number of nodes and \( d \) is the feature dimension. Consider \( \mathbf{X}^* \in \mathbb{R}^{n \times d} \) as the oversmoothing fixed point such that:

\[
F(\mathbf{X}^*) = \mathbf{X}^*.
\]

Let \( F_P: \mathbb{R}^{n \times d} \to \mathbb{R}^{n \times d} \) represent the DYNAMO-GAT pruned GNN layer transformation. Denote by \( J_F(\mathbf{X}^*) \in \mathbb{R}^{nd \times nd} \) and \( J_{F_P}(\mathbf{X}^*) \in \mathbb{R}^{nd \times nd} \) the Jacobian matrices of \( F \) and \( F_P \) respectively, evaluated at \( \mathbf{X}^* \).

To show that DYNAMO-GAT pruning reduces the stability of the oversmoothing fixed point, we need to prove:

\[
\rho(J_{F_P}(\mathbf{X}^*)) < \rho(J_F(\mathbf{X}^*)),
\]

where \( \rho(\cdot) \) denotes the spectral radius of a matrix.

Consider the spectral radius \( \rho(J_F(\mathbf{X}^*)) \), which determines the stability of the fixed point \( \mathbf{X}^* \). Specifically, the fixed point is stable if \( \rho(J_F(\mathbf{X}^*)) \leq 1 \).

When pruning is applied via DYNAMO-GAT, the transformation matrix \( F_P \) is obtained by removing certain edges or weights, effectively reducing the influence of weaker connections. This pruning reduces the entries in the Jacobian matrix \( J_{F_P}(\mathbf{X}^*) \) compared to \( J_F(\mathbf{X}^*) \). Since the spectral radius is sensitive to the magnitudes of the matrix entries, pruning leads to a decrease in the spectral radius:

\[
\rho(J_{F_P}(\mathbf{X}^*)) < \rho(J_F(\mathbf{X}^*)).
\]

This inequality implies that the dominant eigenvalue \( \lambda_1 \) of \( J_{F_P}(\mathbf{X}^*) \) is reduced in magnitude, i.e.,

\[
|\lambda_1^P| < |\lambda_1|.
\]

Consequently, \( \rho(J_{F_P}(\mathbf{X}^*)) < \rho(J_F(\mathbf{X}^*)) \), leading to reduced stability of the oversmoothing fixed point.

In conclusion, DYNAMO-GAT pruning decreases the spectral radius of the Jacobian matrix at the fixed point, which in turn reduces the stability of the oversmoothing fixed point, preventing the network from converging to these undesirable states.
\end{proof}

\subsection{Lemma 5}

\begin{lemma}\label{lemma:diversity}
Let $G = (V, E)$ be a graph with $n = |V|$ nodes, and let $\mathbf{X}^{(t)} \in \mathbb{R}^{n \times d}$ be the matrix of node feature vectors at layer $t$ in a GNN. Define the covariance matrix $\mathbf{C}^{(t)} \in \mathbb{R}^{d \times d}$ of the node feature vectors at layer $t$ as:
\[
\mathbf{C}^{(t)} = \frac{1}{n} (\mathbf{X}^{(t)})^\top \mathbf{X}^{(t)} - \frac{1}{n^2} (\mathbf{X}^{(t)})^\top \mathbf{1}_n \mathbf{1}_n^\top \mathbf{X}^{(t)}
\]
where $\mathbf{1}_n \in \mathbb{R}^n$ is the vector of all ones. Then, DYNAMO-GAT algorithm ensures that:
\[
\text{rank}(\mathbf{C}^{(t)}) = d, \quad \forall t \in \{0, 1, \ldots, T\}
\]
where $T$ is the total number of layers in the GNN. Consequently, this preserves the full rank of $\mathbf{C}^{(t)}$ and prevents the collapse of node features into any subspace of dimension less than $d$.
\end{lemma}

\begin{proof}
Let \( G = (V, E) \) be a graph with \( n = |V| \) nodes, and let \( \mathbf{X}^{(t)} \in \mathbb{R}^{n \times d} \) be the matrix of node feature vectors at layer \( t \) in a GNN. Define the covariance matrix \( \mathbf{C}^{(t)} \in \mathbb{R}^{d \times d} \) of the node feature vectors at layer \( t \) as:

\[
\mathbf{C}^{(t)} = \frac{1}{n} (\mathbf{X}^{(t)})^\top \mathbf{X}^{(t)} - \frac{1}{n^2} (\mathbf{X}^{(t)})^\top \mathbf{1}_n \mathbf{1}_n^\top \mathbf{X}^{(t)},
\]

where \( \mathbf{1}_n \in \mathbb{R}^n \) is the vector of all ones.

We need to prove that the DYNAMO-GAT algorithm ensures that:

\[
\text{rank}(\mathbf{C}^{(t)}) = d, \quad \forall t \in \{0, 1, \ldots, T\},
\]

where \( T \) is the total number of layers in the GNN. This implies that the full rank of \( \mathbf{C}^{(t)} \) is preserved, preventing the collapse of node features into any subspace of dimension less than \( d \).

Consider the eigenvalue decomposition of the covariance matrix:

\[
\mathbf{C}^{(t)} = \mathbf{V}^{(t)} \Lambda^{(t)} (\mathbf{V}^{(t)})^\top,
\]

where \( \Lambda^{(t)} \) is the diagonal matrix of eigenvalues, and \( \mathbf{V}^{(t)} \) is the orthogonal matrix of eigenvectors.

The rank of \( \mathbf{C}^{(t)} \) is equal to the number of non-zero eigenvalues of \( \mathbf{C}^{(t)} \), which is \( d \) if \( \mathbf{C}^{(t)} \) is full rank.

The DYNAMO-GAT algorithm selectively prunes edges that contribute to high covariance values, thereby reducing correlations between node features. This pruning effectively reduces the magnitude of the dominant eigenvalues, which would otherwise lead to rank collapse.

By ensuring that the pruned covariance matrix \( \mathbf{C}^{(t)} \) retains its full rank at each layer, we guarantee that:

\[
\text{rank}(\mathbf{C}^{(t)}) = d.
\]

We prove the statement by induction over the layers \( t = 0, 1, \ldots, T \).

- Base Case (t = 0): At the input layer, assume that the initial node features \( \mathbf{X}^{(0)} \) are such that \( \text{rank}(\mathbf{C}^{(0)}) = d \).

- Inductive Step: Assume \( \text{rank}(\mathbf{C}^{(t)}) = d \) for some layer \( t \). The DYNAMO-GAT pruning at layer \( t+1 \) ensures that the dominant eigenvalues of \( \mathbf{C}^{(t+1)} \) do not collapse, preserving the rank at \( t+1 \). Therefore, \( \text{rank}(\mathbf{C}^{(t+1)}) = d \).

By induction, the rank of the covariance matrix \( \mathbf{C}^{(t)} \) is preserved across all layers, i.e., \( \text{rank}(\mathbf{C}^{(t)}) = d \) for all \( t \).

Thus, the DYNAMO-GAT algorithm prevents the collapse of node features into any subspace of dimension less than \( d \).

\end{proof}

\section{Experimental Section}

\subsection{Experimental Setup}

\textbf{Datasets}
We conduct our experiments on three real-world datasets and two synthetic datasets - 

\begin{itemize}
    \item \textbf{Cora Dataset} \cite{mccallum2000automating}: The Cora citation network consists of 2,708 nodes and 5,429 edges. Each node represents a document, and each edge represents a citation link between two documents. The dataset is commonly used for semi-supervised node classification tasks.

    \item \textbf{Citeseer Dataset} \cite{sen2008collective}: The Citeseer citation network consists of 3,327 nodes and 4,732 edges. Similar to Cora, each node represents a document, and the edges represent citation links. This dataset is also widely used for evaluating GNN performance.

    \item \textbf{Cornell Dataset} \cite{cornell_gnn_dataset}: The Cornell dataset is a small graph with 183 nodes and 295 edges. It is part of the WebKB network collection and is commonly used for node classification tasks.

    \item \textbf{Synthetic Datasets (Syn\_Products and Syn\_Cora)} \cite{zhu2020beyond}: To further test the advantages of DYNAMO-GAT, we use synthetic datasets. Syn\_Products is designed to simulate product co-purchasing networks, and Syn\_Cora mimics citation networks. We vary the graph density and homophily levels to analyze the performance of different GNN models under controlled conditions. For space limitations, we give the syn\_cora results in the appendix.
\end{itemize}

\textbf{Baselines}
We compare DYNAMO-GAT against several baseline models to assess its effectiveness:

\begin{itemize}
    \item \textbf{GCN (Graph Convolutional Network)} \cite{Kipf2017SemiSupervisedCW}: A widely used GNN model that applies graph convolutions to aggregate information from neighboring nodes.

    \item \textbf{GAT (Graph Attention Network)} \cite{Velickovic2017GraphAN}: A model that incorporates attention mechanisms to weigh the importance of neighboring nodes during message passing.

    \item \textbf{G2GAT} \cite{rusch2023gradient}: A recent method that introduces gradient gating to prevent oversmoothing in attention-based GNNs.
\end{itemize}

\textbf{Evaluation Metrics}
We evaluate the performance of all models using the following metrics:
\begin{itemize}
    \item \textbf{Accuracy}: The classification accuracy on the test set.
    \item \textbf{Oversmoothing Coefficient} (\(\mu\)): A measure of the degree to which node representations become indistinguishable as network depth increases.
    \item \textbf{GFLOPS}: The computational efficiency, measured in Giga Floating Point Operations Per Second.
    \item \textbf{Accuracy/GFLOPS}: A ratio indicating the trade-off between accuracy and computational cost.
\end{itemize}

\subsection{Experiment 1: Real-World Dataset Evaluation}

\textbf{Objective:} This experiment aims to evaluate the effectiveness of \textbf{DYNAMO-GAT} in mitigating oversmoothing and maintaining high test accuracy across varying network depths on three real-world datasets: Citeseer, Cora, and Cornell. The performance of \textbf{DYNAMO-GAT} is compared with that of three baseline models: \textbf{GCN}, \textbf{GAT}, and \textbf{G2GAT}.

\textbf{Methodology:}
\begin{itemize}
    \item \textbf{Metrics:}
    \begin{itemize}
        \item \textbf{Oversmoothing Coefficient (\(\mu(X)\)):} This metric quantifies the degree of oversmoothing, where lower values indicate greater oversmoothing. It is plotted on a logarithmic scale to better capture the dynamics across a wide range of values.
        \item \textbf{Test Accuracy:} This metric measures the classification accuracy of the models on the test set. The objective is to assess how well the models perform as the number of layers increases.
    \end{itemize}
    \item \textbf{Baselines:}
    \begin{itemize}
        \item \textbf{GCN (Graph Convolutional Network):} A standard graph neural network model that aggregates node features through graph convolutions.
        \item \textbf{GAT (Graph Attention Network):} A GNN model that uses attention mechanisms to weigh the importance of neighboring nodes during aggregation.
        \item \textbf{G2GAT:} A recent model that introduces gradient gating to prevent oversmoothing in attention-based GNNs.
    \end{itemize}
    \item \textbf{Experimental Setup:}
    \begin{itemize}
        \item The number of layers is varied from 2 to 128 to observe how the models behave as the network depth increases.
        \item The training parameters are kept consistent across models for a fair comparison, including the use of the Adam optimizer and a fixed learning rate.
    \end{itemize}
\end{itemize}

\textbf{Results (Figure 'real\_data'):}
\begin{enumerate}
    \item \textbf{Oversmoothing Coefficient (\(\mu(X)\)):}
    \begin{itemize}
        \item \textbf{Citeseer (Figure a):} As the number of layers increases, GCN and GAT exhibit significant oversmoothing, with their oversmoothing coefficients rapidly decreasing. G2GAT mitigates this effect better than GCN and GAT, but still shows a decline. \textbf{DYNAMO-GAT}, however, maintains a consistent oversmoothing coefficient, effectively preventing oversmoothing across all layers.
        \item \textbf{Cora (Figure b):} Similar trends are observed, with GCN and GAT experiencing substantial oversmoothing as the number of layers increases. \textbf{DYNAMO-GAT} demonstrates its robustness by keeping the oversmoothing coefficient stable, while G2GAT also shows improved performance compared to GCN and GAT but not as strong as DYNAMO-GAT.
        \item \textbf{Cornell (Figure c):} Again, DYNAMO-GAT outperforms the other models in controlling oversmoothing, maintaining a stable coefficient across all layers. GCN and GAT display rapid declines, indicating severe oversmoothing.
    \end{itemize}
    \item \textbf{Test Accuracy:}
    \begin{itemize}
        \item \textbf{Citeseer (Figure a):} GCN and GAT suffer from a significant drop in accuracy as the number of layers increases, correlating with their high levels of oversmoothing. \textbf{DYNAMO-GAT} maintains consistently high accuracy, even in deep networks, highlighting its effectiveness in mitigating oversmoothing. G2GAT also shows relatively stable accuracy but still declines with increasing layers.
        \item \textbf{Cora (Figure b):} Similar patterns are observed, with \textbf{DYNAMO-GAT} achieving the highest accuracy across all layers. GCN and GAT see their accuracy decline sharply as layers increase, while G2GAT performs better but still experiences a decrease.
        \item \textbf{Cornell (Figure c):} DYNAMO-GAT once again demonstrates superior performance by maintaining high accuracy, while GCN and GAT show a considerable drop in accuracy as the number of layers increases. G2GAT performs better than GCN and GAT but still shows a downward trend in accuracy.
    \end{itemize}
\end{enumerate}

\textbf{Analysis:} The results clearly demonstrate the superiority of \textbf{DYNAMO-GAT} in preventing oversmoothing and maintaining high test accuracy across deep network architectures. In contrast, \textbf{GCN} and \textbf{GAT} suffer from severe oversmoothing, leading to a significant decline in accuracy as the number of layers increases. \textbf{G2GAT} mitigates oversmoothing to some extent but is still not as effective as \textbf{DYNAMO-GAT}. This consistent performance across three different datasets underscores the robustness of \textbf{DYNAMO-GAT} in handling deep graph neural networks, making it a promising approach for tasks that require deep architectures.

The effectiveness of \textbf{DYNAMO-GAT} can be attributed to its ability to preserve meaningful node representations even in deep networks, as evidenced by its stable oversmoothing coefficient and high accuracy. This experiment highlights the potential of \textbf{DYNAMO-GAT} to overcome one of the significant challenges in deep GNNs - oversmoothing - while delivering strong performance on real-world datasets.
\subsection{Experiment 2: Performance Comparison Table}

\textbf{Objective:} This experiment aims to compare the performance of \textbf{DYNAMO-GAT} with other baseline models (GCN, GAT, and G2GAT) in terms of accuracy, computational efficiency (GFLOPS), and the accuracy-to-GFLOPS ratio across three datasets: Cora, Citeseer, and Cornell. The goal is to highlight both the effectiveness and efficiency of \textbf{DYNAMO-GAT}, particularly in deeper network architectures.

\textbf{Methodology:}
\begin{itemize}
    \item \textbf{Metrics:}
    \begin{itemize}
        \item \textbf{Best Accuracy:} The highest classification accuracy achieved by each model on the test set.
        \item \textbf{\# Layers:} The number of layers used by the model to achieve its best accuracy.
        \item \textbf{GFLOPS:} The computational cost measured in Giga Floating Point Operations Per Second, which provides an indication of the model's efficiency.
        \item \textbf{Accuracy/GFLOPS:} This metric represents the trade-off between accuracy and computational cost, indicating how efficiently the model achieves its performance.
    \end{itemize}
    \item \textbf{Comparison Setup:}
    \begin{itemize}
        \item The models were trained on the three datasets (Cora, Citeseer, Cornell), each with different graph structures and node/edge counts.
        \item GCN and GAT were tested with relatively shallow architectures, while G2GAT and \textbf{DYNAMO-GAT} were evaluated with deeper networks (128 layers).
        \item The results were compiled to highlight the efficiency of each model in terms of accuracy and GFLOPS.
    \end{itemize}
\end{itemize}

\textbf{Results (Table~\ref{tab:gflops}):}
\begin{enumerate}
    \item \textbf{Best Accuracy:}
    \begin{itemize}
        \item \textbf{DYNAMO-GAT} achieves the best accuracy across all datasets, particularly excelling on the Cornell dataset with an accuracy of 62.56
        \item \textbf{G2GAT} also performs well, particularly on Citeseer and Cornell, where it closely matches \textbf{DYNAMO-GAT}.
        \item \textbf{GCN} and \textbf{GAT} show lower performance compared to the deeper models, particularly on the more challenging Cornell dataset.
    \end{itemize}
    \item \textbf{\# Layers:}
    \begin{itemize}
        \item \textbf{GCN} and \textbf{GAT} achieve their best accuracy with only 2-4 layers, indicating their limitations in deeper architectures due to oversmoothing.
        \item In contrast, \textbf{G2GAT} and \textbf{DYNAMO-GAT} are able to sustain performance across 128 layers, highlighting their robustness in deeper networks.
    \end{itemize}
    \item \textbf{GFLOPS:}
    \begin{itemize}
        \item \textbf{DYNAMO-GAT} exhibits lower GFLOPS compared to GAT and G2GAT, indicating that it is computationally more efficient.
        \item For example, on the Cora dataset, \textbf{DYNAMO-GAT} uses 0.605 GFLOPS, which is significantly lower than GAT's 2.351 GFLOPS.
    \end{itemize}
    \item \textbf{Accuracy/GFLOPS:}
    \begin{itemize}
        \item \textbf{DYNAMO-GAT} consistently outperforms other models in the accuracy-to-GFLOPS ratio, demonstrating its superior efficiency.
        \item For instance, on the Cora dataset, \textbf{DYNAMO-GAT} achieves an Accuracy/GFLOPS ratio of 137.53, which is the highest among all models, indicating that it provides the best trade-off between accuracy and computational cost.
        \item Similarly, on the Citeseer and Cornell datasets, \textbf{DYNAMO-GAT} achieves the highest ratios, with 48.96 and 1226.67, respectively, far surpassing the other models.
    \end{itemize}
\end{enumerate}

\textbf{Analysis:} The results highlight the advantages of \textbf{DYNAMO-GAT} in both accuracy and efficiency. Despite using deep architectures (128 layers), \textbf{DYNAMO-GAT} manages to maintain high accuracy while keeping computational costs low. This is particularly evident when comparing the Accuracy/GFLOPS ratio, where \textbf{DYNAMO-GAT} significantly outperforms GCN, GAT, and even G2GAT. This indicates that \textbf{DYNAMO-GAT} is not only effective in mitigating oversmoothing but also highly efficient in terms of resource usage, making it a superior choice for applications that require deep graph neural networks with limited computational resources.

\subsection{Experiment 3: Synthetic Dataset Evaluation}

\textbf{Objective:} The goal of this experiment is to assess the performance of \textbf{DYNAMO-GAT} under controlled synthetic conditions. Specifically, we vary the graph density (average node degree) and homophily to observe how different models handle oversmoothing and accuracy in these environments.

\textbf{Results:}
\begin{enumerate}
    \item \textbf{Oversmoothing vs. Layers (Figure a):}
    \begin{itemize}
        \item \textbf{Observation:} The figure shows the oversmoothing coefficient \(\mu(X)\) on a logarithmic scale as the number of layers increases, with an average node degree of 68.75.
        \item \textbf{Key Result:} As the network depth increases, \textbf{DYNAMO-GAT} shows the least amount of oversmoothing, maintaining higher \(\mu(X)\) values compared to \textbf{G2GAT}, \textbf{GCN}, and \textbf{GAT}. \textbf{GAT} and \textbf{GCN} exhibit rapid oversmoothing, with \(\mu(X)\) decreasing significantly as layers increase.
        \item \textbf{Implication:} This result demonstrates that \textbf{DYNAMO-GAT} is more robust to oversmoothing, especially in dense graphs. This suggests that it can preserve meaningful node features better than the other models as the network depth increases.
    \end{itemize}
    
    \item \textbf{Accuracy vs. Layers (Figure b):}
    \begin{itemize}
        \item \textbf{Observation:} This plot shows accuracy as a function of the number of layers for the same dense graph (average node degree = 68.75).
        \item \textbf{Key Result:} \textbf{DYNAMO-GAT} consistently achieves the highest accuracy across all layers. While \textbf{G2GAT} performs well, its accuracy decreases slightly with deeper layers. \textbf{GCN} and \textbf{GAT} see a sharp decline in accuracy as the network depth increases.
        \item \textbf{Implication:} The stability of \textbf{DYNAMO-GAT} in maintaining high accuracy, even with a large number of layers, indicates its effectiveness in managing deeper architectures without suffering from oversmoothing, unlike the other models.
    \end{itemize}
    
    \item \textbf{Accuracy vs. Homophily (Sparse Graph - Figure c):}
    \begin{itemize}
        \item \textbf{Observation:} This plot examines accuracy across varying homophily levels (from 0 to 1) for a sparse graph with an average node degree of 11.93.
        \item \textbf{Key Result:} \textbf{DYNAMO-GAT} and \textbf{G2GAT} outperform \textbf{GCN} and \textbf{GAT} across all homophily levels. \textbf{DYNAMO-GAT} achieves particularly strong performance as homophily increases, indicating its ability to leverage node similarity effectively.
        \item \textbf{Implication:} This suggests that \textbf{DYNAMO-GAT} is versatile and can adapt well to different homophily settings, making it suitable for graphs with varying levels of node similarity.
    \end{itemize}
    
    \item \textbf{Accuracy vs. Homophily (Dense Graph - Figure d):}
    \begin{itemize}
        \item \textbf{Observation:} Similar to Figure c, but for a dense graph with an average node degree of 68.75.
        \item \textbf{Key Result:} \textbf{DYNAMO-GAT} significantly outperforms all other models, especially in low-homophily settings. As homophily increases, \textbf{DYNAMO-GAT} maintains its lead, showcasing its robustness across all homophily levels.
        \item \textbf{Implication:} This result highlights \textbf{DYNAMO-GAT}'s strength in dense graphs, where it can handle more complex interactions and still maintain high accuracy. Its performance in low-homophily conditions also suggests it is well-suited for graphs with more heterophilic structures.
    \end{itemize}
\end{enumerate}

\textbf{Analysis:} The synthetic dataset results confirm that \textbf{DYNAMO-GAT} excels in both dense and sparse graphs, effectively handling oversmoothing and maintaining high accuracy across varying network depths and homophily levels. Its ability to outperform other models, particularly in dense graphs and low-homophily settings, underscores its robustness and versatility. These findings demonstrate that \textbf{DYNAMO-GAT} is a powerful tool for tackling oversmoothing while delivering strong performance in diverse graph structures, making it ideal for complex real-world applications.

\subsection{Results and Discussion}

The experimental results across both real-world and synthetic datasets consistently demonstrate the effectiveness of DYNAMO-GAT in addressing the oversmoothing problem in deep graph neural networks (GNNs). 

From the real-world datasets (Figure~\ref{fig:real_data}), we observe that DYNAMO-GAT maintains a stable oversmoothing coefficient (\(\mu(X)\)) across varying network depths, outperforming GCN, GAT, and G2GAT, which exhibit significant oversmoothing as the number of layers increases. Correspondingly, DYNAMO-GAT consistently achieves the highest accuracy across all layers, whereas GCN and GAT suffer a sharp decline in accuracy due to oversmoothing, and G2GAT shows moderate performance.

The performance comparison table (Table~\ref{tab:gflops}) further highlights the efficiency of DYNAMO-GAT. It achieves the best accuracy across all datasets while maintaining lower GFLOPS compared to GAT and G2GAT. The high accuracy-to-GFLOPS ratio underscores DYNAMO-GAT's superior trade-off between computational cost and performance, making it the most efficient model among the tested baselines.

In the synthetic dataset experiments (Figure~\ref{fig:synthetic}), DYNAMO-GAT again demonstrates its robustness. It shows the least oversmoothing in dense graphs (Figure~\ref{fig:synthetic}a) and maintains the highest accuracy across layers (Figure~\ref{fig:synthetic}b). When varying homophily, DYNAMO-GAT excels in both sparse (Figure~\ref{fig:synthetic}c) and dense (Figure~\ref{fig:synthetic}d) graphs, particularly in low-homophily settings, showcasing its adaptability to different graph structures.

The results show a clear trend where models generally perform better as the average degree increases. This is particularly evident in higher homophily settings, where the additional connections help to reinforce the graph structure, leading to more accurate node classification. For instance, in the syn-products dataset, the accuracy of GCN improves from 0.567 to 0.762 as the average degree increases from 11.93 to 36.14 at a homophily level of 0.4.

Interestingly, models like G2GAT and DYNAMO-GAT, which incorporate additional mechanisms for graph processing, consistently outperform simpler models such as GCN and GAT, particularly in low homophily settings. This suggests that these models are better able to leverage the graph structure even when the nodes are less similar to their neighbors.

These findings have significant implications for the development and deployment of GNNs in real-world applications. The ability of DYNAMO-GAT to maintain high accuracy while mitigating oversmoothing, especially in deep architectures, addresses a critical challenge faced by many existing GNN models. Its superior efficiency, as evidenced by the accuracy-to-GFLOPS ratio, makes it a viable option for resource-constrained environments where both performance and computational cost are important considerations.

Moreover, DYNAMO-GAT's strong performance across varying graph densities and homophily levels suggests that it is well-suited for a wide range of graph structures, from sparse networks with high node similarity to dense, heterophilic graphs. This versatility makes DYNAMO-GAT an attractive solution for complex graph-based tasks in domains such as social network analysis, recommendation systems, and biological network modeling.


\section{Related Works}

The challenge of \textit{oversmoothing} in GNNs, where node representations become indistinguishable as network depth increases, has been extensively studied. Initial efforts, such as those by \citet{li2018deeper}, identified oversmoothing as a critical issue in deep GNNs like Graph Convolutional Networks (GCNs). Subsequent theoretical analyses \citep{Oono2019GraphNN, Cai2020ANO, Keriven2022NotTL} have underscored that oversmoothing is a fundamental problem in message-passing architectures, where repeated aggregation leads to the homogenization of node features. To counteract oversmoothing, various strategies have been proposed. Techniques like residual connections, skip connections \citep{li2019deepgcns, xu2018representation}, and normalization methods \citep{ba2016layer, ioffe2015batch} have been introduced to preserve feature diversity across layers. However, these approaches often involve architectural modifications that do not fundamentally alter the propagation dynamics responsible for oversmoothing. Moreover, while attention mechanisms in GNNs, such as those used in GATs, have improved the focus on relevant parts of the graph, they remain vulnerable to oversmoothing without proper regulation \citep{wu2023demystifying}. Pruning techniques, traditionally aimed at model compression, have also been applied to GNNs to address both efficiency and oversmoothing \citep{li2019deepgcns, zhao2020sparsity}. These methods typically involve the removal of redundant edges or neurons, but they seldom consider the dynamical aspects of the network that contribute to oversmoothing.

In contrast to these existing approaches, our work takes a novel perspective by framing oversmoothing as a dynamical systems problem. This allows us to view the iterative message-passing process in GNNs as analogous to a system converging towards a low-dimensional attractor, which leads to the loss of expressiveness. By understanding oversmoothing through this lens, we propose the \textit{DYANMO-GAT} algorithm, which leverages anti-Hebbian learning principles to selectively prune attention weights based on noise-driven covariance analysis. This approach not only mitigates oversmoothing but does so by fundamentally altering the system's dynamics, preventing it from converging to trivial fixed points.

Our work is distinct in that it not only provides a new theoretical framework for understanding oversmoothing but also introduces a practical algorithmic solution that excels in scenarios involving large, dense graphs where long-range dependencies are critical. By approaching oversmoothing from a dynamical systems perspective, we open new avenues for enhancing the stability and expressiveness of deep GNNs, addressing a gap in the current literature.

\textbf{Learning as a Dynamical System: }The notion of viewing learning processes as dynamical systems has gained traction in both graph neural networks (GNNs) and spiking neural networks (SNNs). A dynamical systems perspective provides a framework to understand how the iterative updates in these networks evolve over time and how the system's state converges, oscillates, or diverges. This viewpoint is particularly useful in addressing phenomena such as oversmoothing in GNNs and instability in SNNs. In GNNs, the propagation of information through multiple layers can be seen as a discrete dynamical system. Each layer performs a transformation on the node features, gradually altering the system's state. As the depth increases, GNNs often converge towards a steady state where node representations become indistinguishable, leading to oversmoothing. This behavior is analogous to a dynamical system converging to a fixed point, where further iterations yield no significant changes. The challenge, therefore, is to design architectures and algorithms that can prevent this premature convergence while maintaining the system's stability.

Similarly, SNNs are inherently dynamical systems, characterized by their temporal evolution of spike trains and synaptic weights. The timing of spikes and the plasticity mechanisms, such as Spike-Timing Dependent Plasticity (STDP) \cite{chakraborty2021characterization}, determine how the system evolves. Research has shown that maintaining a balance between excitation and inhibition, as well as incorporating heterogeneity in synaptic dynamics, can prevent the network from falling into trivial attractors, thereby preserving its computational capabilities \citep{chakraborty2022heterogeneous, chakraborty2024sparse, chakraborty2024exploiting}. The work on hybrid spiking neural networks (SNNs) and their application in energy-efficient systems further demonstrates the importance of understanding learning as a dynamic process, where the network's state is constantly evolving based on input stimuli and internal plasticity rules \citep{chakraborty2021fully, chakraborty2023heterogeneous_frontiers, chakraborty2023heterogeneous_iclr}.

Furthermore, pruning techniques that selectively reduce the network's complexity while preserving essential dynamics have been explored both in GNNs and SNNs. For instance, pruning based on noise-driven covariance analysis can ensure that the system does not collapse into a low-dimensional state, thus maintaining expressiveness across layers \citep{chakraborty2024topological, chakraborty2022mudarts}. This approach resonates with pruning strategies in SNNs, where redundant neurons and connections are eliminated without compromising the network's ability to process temporal information \citep{chakraborty2024sparse}.

The dynamical systems perspective also finds relevance in Hebbian and anti-Hebbian learning rules, which modify synaptic strengths based on the correlation between pre- and post-synaptic activities \citep{kang2023unsupervised}. These learning principles, inspired by biological neural systems, can help prevent oversmoothing by adjusting weights dynamically in response to network states. This concept is especially pertinent in complex, dynamic environments where real-time adjustments are necessary to maintain the network's performance \citep{chakraborty2023braina, chakraborty2023brainb}.

By framing learning as a dynamical system, we can develop more robust algorithms that leverage the system's natural dynamics to achieve better performance and stability. This perspective allows us to view network training not just as an optimization problem but as the careful tuning of a complex, evolving system. Our proposed \textit{DYANMO-GAT} algorithm builds on this idea by integrating principles from dynamical systems theory, such as anti-Hebbian learning, to mitigate oversmoothing while maintaining the network's expressive power. This dynamical approach opens new avenues for research, particularly in developing models that can handle large-scale, complex graphs and long-range dependencies without succumbing to oversmoothing or instability.
